\documentclass[conference]{IEEEtran}
\IEEEoverridecommandlockouts
\usepackage{amsmath,amssymb,amsfonts}
\usepackage{algorithmic}
\usepackage{graphicx}
\usepackage{textcomp}
\usepackage{xcolor}
\def\BibTeX{{\rm B\kern-.05em{\sc i\kern-.025em b}\kern-.08em
    T\kern-.1667em\lower.7ex\hbox{E}\kern-.125emX}}

\usepackage{times}
\usepackage{helvet}
\usepackage{courier}
\usepackage{hyperref}
\usepackage{graphicx}
\usepackage{caption}
\usepackage{algorithm}
\usepackage{algorithmic}
\usepackage{newfloat}
\usepackage{listings}

\usepackage[utf8]{inputenc} 
\usepackage{booktabs}       
\usepackage{nicefrac}       
\usepackage{microtype}      
\usepackage{multirow}

\usepackage{adjustbox}
\usepackage{amsmath}
\usepackage{amsthm}
\usepackage{amssymb}
\usepackage{amsfonts}
\usepackage{thmtools, thm-restate}

\newcommand\scalemath[2]{\scalebox{#1}{\mbox{\ensuremath{\displaystyle #2}}}}

\begin{document}

\title{Goal-Aware Neural SAT Solver}
\author{
\IEEEauthorblockN{Emils Ozolins, Karlis Freivalds, Andis Draguns, \\
Eliza Gaile, Ronalds Zakovskis, Sergejs Kozlovics}
\IEEEauthorblockA{Institute of Mathematics and Computer Science\\
University of Latvia\\
\texttt{ozolinsemils@gmail.com}, \texttt{karlis.freivalds@lumii.lv}}
}

\maketitle

\begin{abstract}

Modern neural networks obtain information about the problem and calculate the output solely from the input values. We argue that it is not always optimal, and the network's performance can be significantly improved by augmenting it with a query mechanism that allows the network at run time to make several solution trials and get feedback on the loss value on each trial. To demonstrate the capabilities of the query mechanism, we formulate an unsupervised (not depending on labels) loss function for Boolean Satisfiability Problem (SAT) and theoretically show that it allows the network to extract rich information about the problem. We then propose a neural SAT solver with a query mechanism called QuerySAT and show that it outperforms the neural baseline on a wide range of SAT tasks.

\end{abstract}

\section{Introduction}

Boolean Satisfiability Problem (SAT) is a significant NP-complete problem with numerous practical applications -- product configuration, hardware verification, and software package management, to name a few. There is no single efficient algorithm that solves every SAT problem, but heuristics have been developed that are sufficient for solving many real-life tasks.

Neural networks have demonstrated superior performance on many complex tasks, e.g., image and language processing, protein folding,  and playing board and computer games. So we may expect that, in the long run, neural networks could also replace handcrafted heuristics for SAT problems. 

There are attempts for applying neural networks to solve SAT by integrating them in contemporary solvers or replacing them altogether. \cite{selsam2018learning} proposes training NeuroSAT architecture with single-bit supervision, predicting whether the SAT instance is satisfiable.
Yet, it does not directly produce variable assignments and requires labels for the training that may be troublesome to obtain. \cite{amizadeh2018learning} introduces unsupervised loss for training a neural network to solve a related Circuit-SAT problem without requiring labels. They use differentiable constraint relaxation to evaluate the network output and penalise the network for violating constraints. That leads to successful training for Circuit-SAT problems.

Usually, neural networks are trained by minimizing the loss which is obtained at the final layer. We ask whether it is beneficial to let the network know the loss it is currently producing and adapting its output correspondingly. That is not possible with the traditional supervised losses since labels are not available at the inference time. Yet, such an approach is sound for an unsupervised loss and, indeed, provides an excellent basis for designing a neural SAT solver. 

In this paper, we introduce a step-wise recurrent neural SAT solver that at each step comes up with a query of variable assignments, evaluates it with an unsupervised loss, and updates its state based on the evaluation results.  At every step, it outputs a target solution that may contain information from all the queries performed so far. 

We prove that the proposed query mechanism employing unsupervised loss function as query evaluator can provide the neural network with information about (a) satisfiability status of a solution if queried at integer points, (b) reveal the structure of the problem instance if queried at fractional (real) points. That provides rich information about the structure, meaning of the problem, and its solution.

\begin{figure}[t]
    \centering
    \includegraphics[width=\columnwidth]{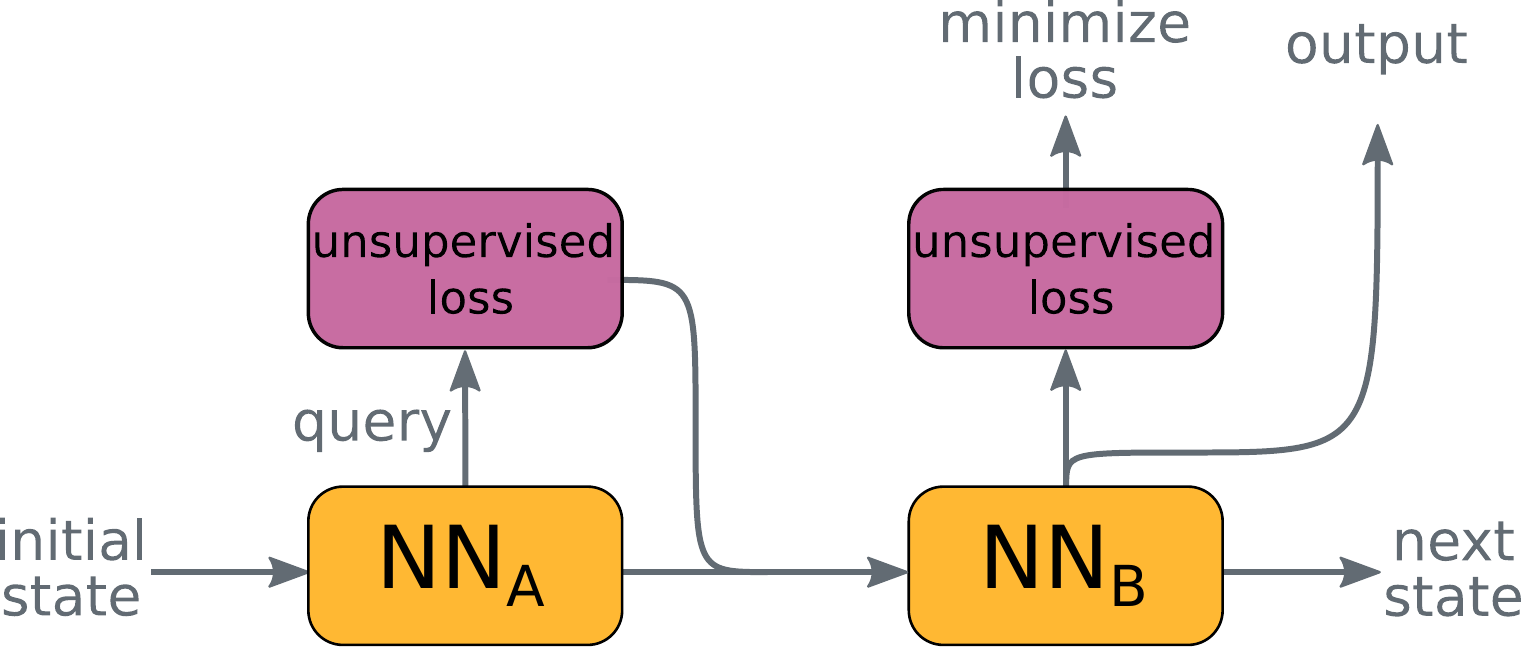}
    \caption{The proposed query mechanism works by producing a query, evaluating it using an unsupervised loss function, and passing the resulting value back to the neural network for interpretation. It allows the model to obtain the structure and meaning of the solvable instance and information about the expected model output. The same unsupervised loss can be used for evaluating the query and for training.
    }
    \label{fig:query_overview}
\end{figure}

The proposed architecture is evaluated on several SAT problems, and it outperforms the neural baseline on all of them. We also compare it with classical SAT solvers (Glucose 4 and GSAT) on 3-SAT and SHA-1 preimage attack tasks. Appendix of this paper is available on \url{https://github.com/LUMII-Syslab/QuerySAT/blob/master/appendix.pdf}.

\section{Query mechanism}\label{sec:querysat}
We hypothesize that allowing a neural network to make several solution trials at the runtime and getting feedback about them can significantly improve network capabilities in some cases. It is easy to implement almost for any neural network, yet depending on the unsupervised loss function, some networks can be especially suitable for such augmentation.

The simplest such implementation (depicted in Fig. \ref{fig:query_overview}) consists of two different neural network layers. The first layer ($NN_A$) is given the current state ($s_{r}$) and it outputs a query ($q$) and some hidden state ($h$). We then use the unsupervised loss function ($loss$) to evaluate the query and obtain evaluation results ($e$). It is important to use a loss function that does not require labels for evaluation, as they may not be available at the test time. The evaluation results ($e$) together with the hidden state ($h$) are then passed as the input for the second layer ($NN_B$) that is responsible for interpreting the evaluated query and producing the next state ($s_{r+1}$) and output logits ($l$).  We can calculate the gradient of the evaluation results ($e$) with respect to query ($q$). The gradient can be optionally added as the third input for the second layer to show the network the direction for decreasing the loss \cite{andrychowicz2016learning}. The query mechanism is conceptually shown in Fig. \ref{fig:query_overview} and can be defined as follows:
\begin{align*}
        q, h &= NN_{A}(s_{r}) \\
        e &= loss(q) \\
        s_{r+1}, l &= NN_{B}(h,\; e,\; \nabla_q e) \\
        \mathcal{L} &= loss(l).
\end{align*}

Calculating output logits and loss $\mathcal{L}$ at every step is not noteworthy per se but merely demonstrates that the same loss function can be used for both -- query evaluation and model training. The training loss in principle could be evaluated only once -- at the network's final layer and can be different from the loss used for the query evaluation. Nevertheless, it is essential not to merge the loss (and logits) used for query with those used for training. Although constructed similarly, their meaning and functioning are different (see Appendix \ref{apx:query_use}).

\section{Query mechanism for SAT}\label{sec:query_mechanism}

We use the Boolean Satisfiability Problem (SAT) as the testbed for validating the benefits of the query mechanism. To this end, we propose an unsupervised loss function for SAT problems in conjunctive normal form and theoretically show that the query mechanism is indeed beneficial -- it gives the model access to the problem structure and contributes to the performance of the model. 

\subsection{Boolean Satisfiability Problem}
Boolean Satisfiability Problem (SAT) questions whether there exists an interpretation (True or False assignments to the variables) that satisfies the given Boolean formula. We undertake a related problem -- finding a set of variable assignments that satisfies the given Boolean formula. Throughout the paper, we stick to common practice \cite{biere2009handbook} and represent SAT formulas solely in a conjunctive normal form (CNF) -- conjunction of one or more clauses where a clause is a disjunction of literals.

Any SAT formula can be naturally represented as a bipartite variables-clauses graph, likewise known as an SAT factor graph \cite{biere2009handbook}. In such a graph, the edge between variables and clauses graph exists whenever the variable is present in the clause. Two types of edges are used to distinguish a variable from its negation.
Variables-clauses graph of SAT formula with $n$ variables and $m$ clauses can be represented as a sparse $n \times m$ adjacency matrix. In order to perform batching, several SAT instances can be placed into a single factor graph yielding a single adjacency matrix for the whole batch.

\subsection{Unsupervised SAT loss}\label{sec:sat_loss}
A common way to train neural networks is by using a supervised loss, such as cross-entropy, which matches the network outputs with the correct labels. But such an approach does not work for training variable assignment for SAT due to several possible satisfying assignments for a single SAT instance. Also, obtaining labels involve SAT solving, which is time-consuming for large instances. Moreover, as we want to integrate the loss function into the neural network as a query mechanism, it has to be differentiable and cannot rely on labels as they are not available at the test time. 

Therefore, we design a differentiable unsupervised loss function that directly optimizes towards finding a satisfiable variable assignment of the Boolean formula without knowing a correct solution. To this end, we relax the Boolean domain to continuous variables in the range $[0, 1]$ where 0 corresponds to False and 1 to True. The vector of all variable values in the assignment is denoted by $x$. The value $V_c$ for a clause $c$ is obtained by multiplying together negations of all literals in the clause and negating the result and the loss value $\mathcal{L_{\phi}}$ for a formula $\phi$ by multiplying all the clause values of $\phi$ together: 
\begin{align*}
    V_c(x) &= 1-\prod_{i \in c^+}(1-x_i)\prod_{i \in c^-}x_i, \\
    \mathcal{L_{\phi}}(x) &= \prod_{c \in \phi}V_c(x), \\
\end{align*}
where $x_i$ is the value of $i$-th variable and $c^+$ gives the set of variables that occur in the clause $c$ in the positive form and $c^-$ in the negated form. 
The values $V_c(x)$ and $\mathcal{L_{\phi}}(x)$ are equal to 1 if and only if $x$ is a satisfying assignment of clause $c$ or formula $\phi$, respectively, and strictly smaller otherwise. So by maximizing $\mathcal{L}_{\phi}(x)$ we can hope to find a satisfying assignment to the variables. But this function is not usable in practice since it often yields zero value in machine precision due to multiplying together many clause values, which are all less than one. Therefore, we use the negative logarithm of $\mathcal{L}_{\phi}$ and minimize it:
\begin{align*}
    \mathcal{L}_{\phi}^{\log}(x) &= -\log(\mathcal{L}_{\phi}(x)) = -\sum_{c \in \phi}\log(V_c(x)).
\end{align*}
Taking the logarithm of the loss does not change its minimum/maximum structure, so minimizing $\mathcal{L}^{\log}_{\phi}$ is equivalent of maximizing $\mathcal{L}_{\phi}$. 

A model trained by minimizing $\mathcal{L}^{\log}_{\phi}$ produces “soft” variable assignments in the range $[0,1]$. We experimentally observe that when the network can find a solution, it is close to binary and rounding the values to the nearest integer produces excellent results. 

Other works \cite{Ryan2020CLN2INV:, pmlr-v97-fischer19a}, albeit in different contexts, have also proposed relaxing SAT formulas to continuous truth values similar to $\mathcal{L}_{\phi}$ loss (one not in the log-space) but faces the same problems as $\mathcal{L}_{\phi}$ if used for learning SAT solutions directly. Therefore, log-space formulation $\mathcal{L}^{\log}_{\phi}$ of the loss  trades capabilities of modelling general Boolean formulas for a loss function that gives non-zero loss for large CNF formulas and can be directly applied for learning their solutions. 

\subsection{Power of the query}
The proposed loss function, when used in the query mechanism, gives the satisfiability status of a solution if queried at the binary points 0 or 1 and reveals the structure of the formula if queried at the intermediate (real) points. To maximize the amount of information gained by the query and to match the internal structure of Graph Neural Network (GNN), we employ queries that return the loss value $V_c(x)$ for each clause at the query point $x$. The power of such query mechanism is formulated in the Theorem \ref{thm:int_points} and \ref{thm:instance_to_be_solved}, respectively. The proofs of theorems are deferred to the Appendix \ref{apx:int_points} and \ref{apx:instance_to_be_solved}.

\begin{restatable}{theorem}{thintpoints}\label{thm:int_points}
For a binary query point $x$ the losses $\mathcal{L}_{\phi}(x)$ or $V_c(x)$ are equal to 1 if the formula $\phi$ or clause $c$ is satisfied and 0 otherwise.
\end{restatable}

\begin{restatable}{theorem}{thsolvedinstances}\label{thm:instance_to_be_solved}
A single query that returns the loss $V_c$ for each clause $c$ is sufficient to uniquely identify the SAT formula $\phi$ to be solved.
\end{restatable}

The proof of theorem \ref{thm:int_points} follows immediately from the definitions of the losses. The proof of theorem \ref{thm:instance_to_be_solved} shows how to create a query point $x$ in such a way that the literals making up the clauses can be uniquely decoded solely from the clause losses. The proof assumes that we know the variable and clause count of the formula, but this assumption can be relaxed by padding all the formulas to some fixed maximum size. A stronger result that a single query returning only the total loss $\mathcal{L}_{\phi}$ (a single real number) might also be shown by using more intricate reasoning, but we did not pursue this direction since we use per-clause loss. 

The construction used in theorem \ref{thm:instance_to_be_solved} assumes sufficiently high precision of the numbers, exceeding the limits of standard floating-point arithmetic. Regarding that, in our implementation, we issue several queries in parallel (up to 128) and organize the computation in multiple recurrent steps in which queries are performed repeatedly and one query can depend on the results of the previous one. Such design significantly enhances the power of queries and relieves the need for high precision.

\section{QuerySAT}
We validate the theory in practice by building a neural SAT solver that we call QuerySAT and evaluate its performance on a wide range of SAT tasks - k-SAT, 3-SAT, 3-Clique, k-Coloring, and SHA-1 preimage attack. Experimental evaluation is performed on a single machine with 16GB RAM and a single Nvidia T4 GPU (16GB) using AdaBelief optimizer \cite{zhuang2020adabelief}. Code for reproducing experiments is implemented in TensorFlow and is available at \url{https://github.com/LUMII-Syslab/QuerySAT}.

\subsection{Model}\label{sec:querysat_model}

\begin{figure*}[ht]
     \centering
     \begin{minipage}[t]{0.32\textwidth}
         \centering
         \includegraphics[width=\textwidth]{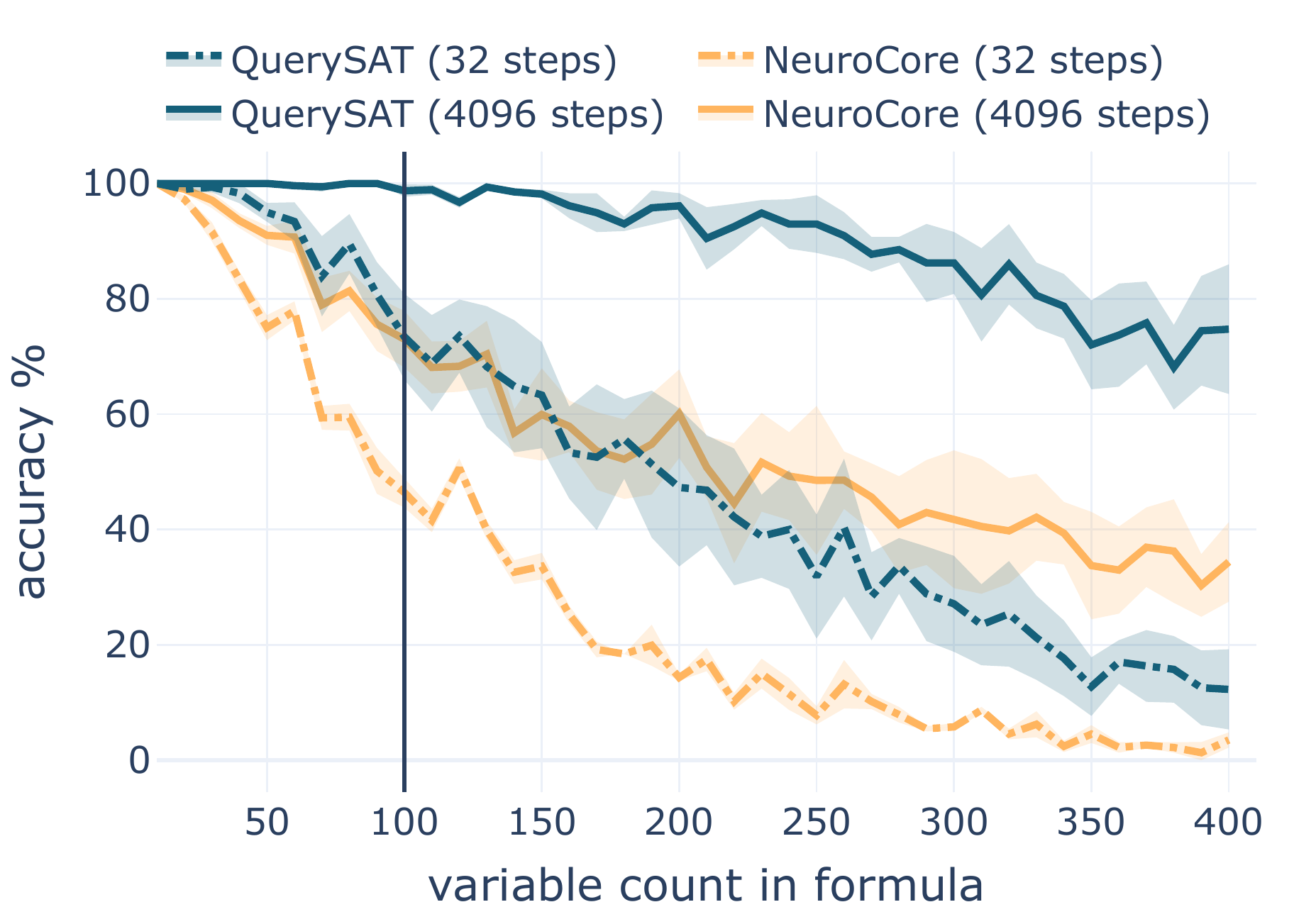}
         \caption{Part of fully solved 3-SAT instances of the test set depending on the variable count. Models were trained with 32 recurrent steps on formulas with up to 100 variables.}
         \label{fig:accuracy_on_datasets}
     \end{minipage}
     \hfill
     \begin{minipage}[t]{0.32\textwidth}
         \centering
         \includegraphics[width=\textwidth]{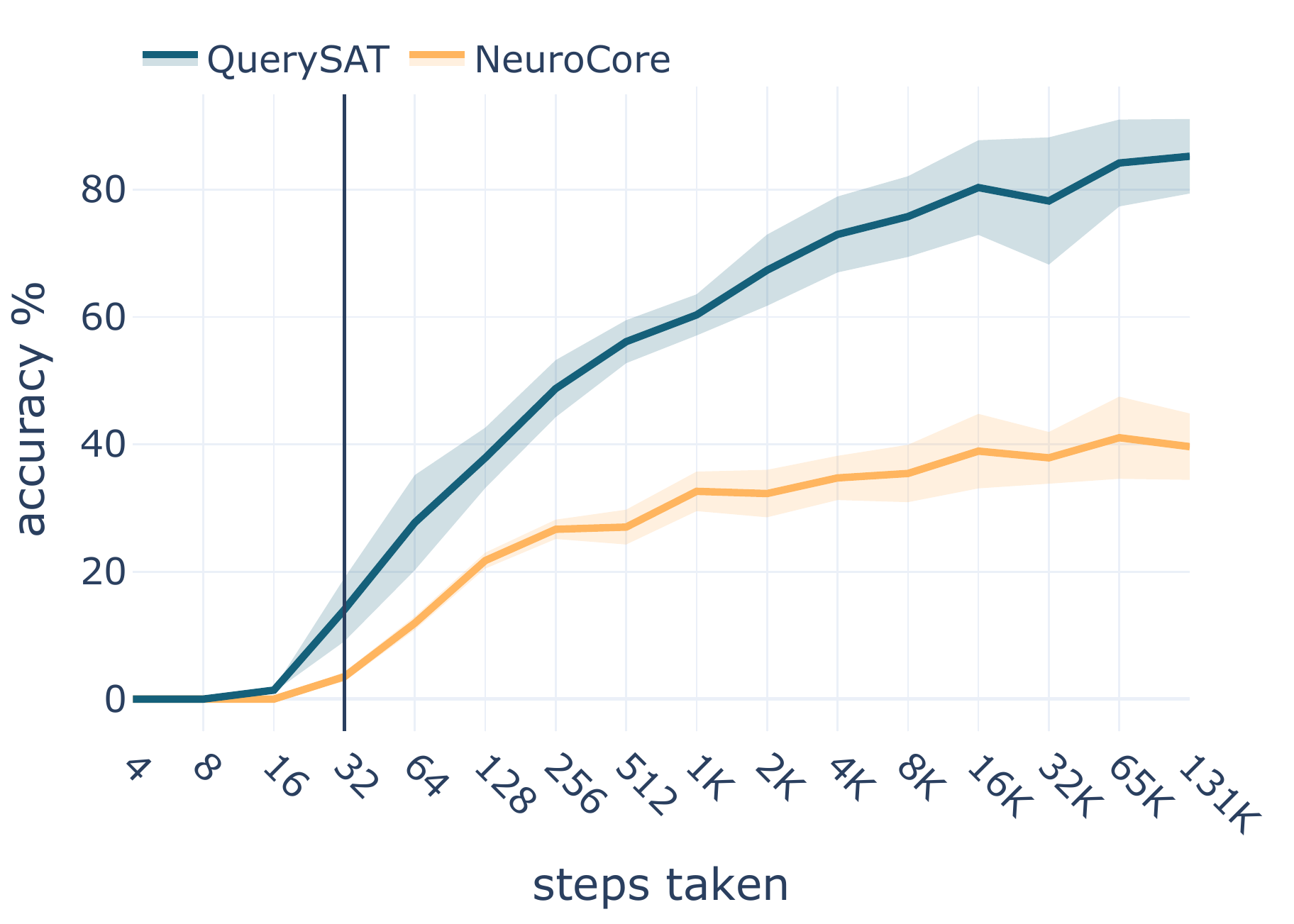}
         \caption{Part of fully solved 3-SAT instances with 400 variables depending on the steps taken at test time. Models were trained with 32 recurrent steps on formulas with up to 100 variables.}
         \label{fig:accuracy_on_steps}
     \end{minipage}
     \hfill
     \begin{minipage}[t]{0.32\textwidth}
         \centering
         \includegraphics[width=\textwidth]{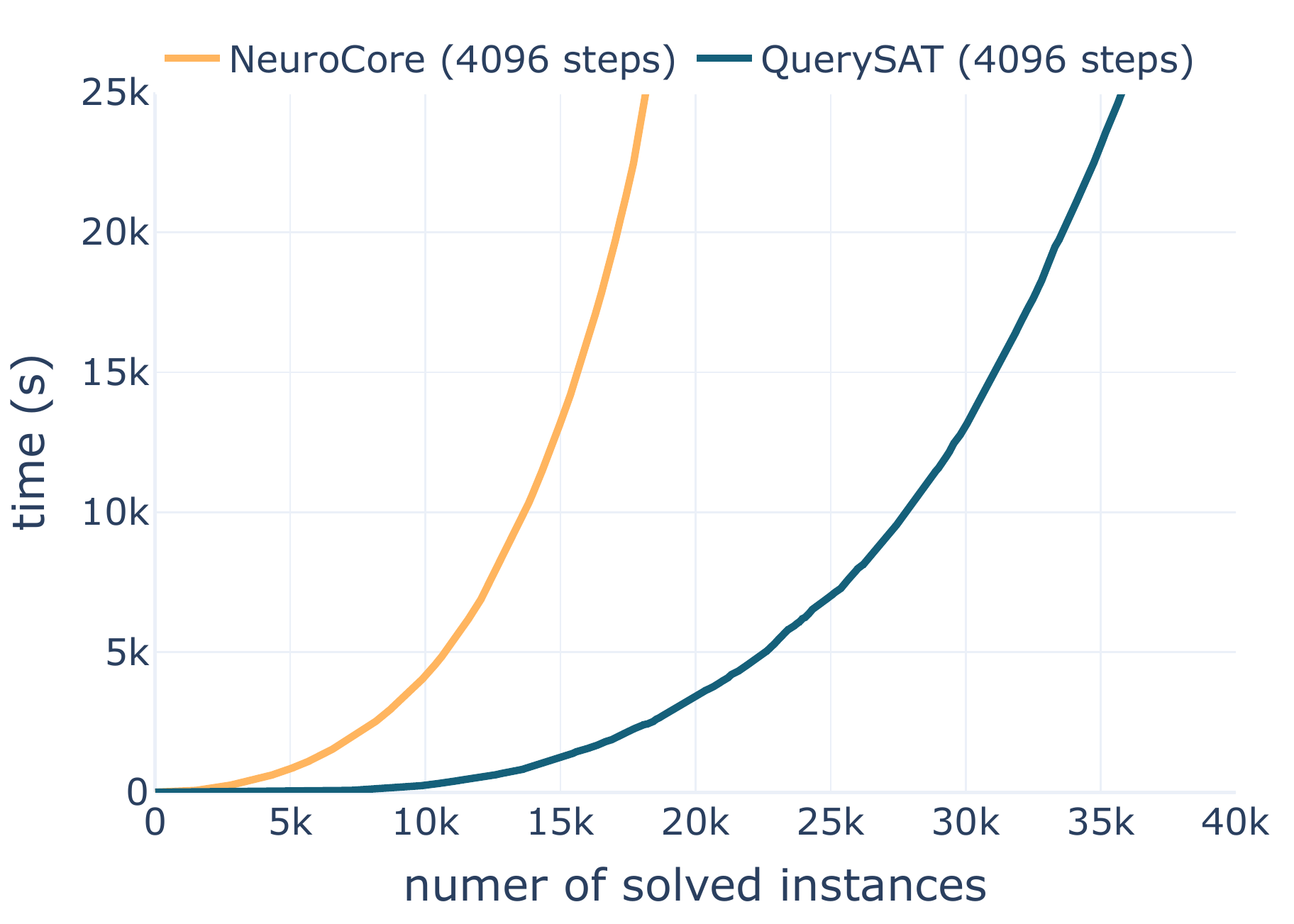}
         \caption{Time vs fully solved instances trained with 32 recurrent steps on 3-SAT instances with 5-100 variables and then tested on instances with 5-405 variables.}
         \label{fig:neural_solver_cactus}
     \end{minipage}
\end{figure*}

QuerySAT is based on a GNN employing the proposed query mechanism and unsupervised SAT loss function. It receives a CNF Boolean formula $\phi$ in the input represented as two adjacency matrices of variables-clauses graphs - $A_{p} \in \{0,1\}^{n \times m}$ and  $A_{n} \in \{0,1\}^{n \times m}$, where $n$ and $m$ are the number of variables and clauses, respectively.  
$A_{p}$ represents all positive variable mentions in the clauses, and $A_{n}$ represents all negated variable mentions. The network outputs a vector $out \in [0,1]^n \subseteq \mathbb{R}^n$ -- a variable assignment. 
QuerySAT works in a step-wise manner with recurrent application of the following graph-based recurrent unit: 
\begin{align*}\label{formula:querysat}
    q_i &= \sigma (\text{MLP}_q(v_i,\; t)) \\
    e_i &= \mathcal{V_{\phi}}(q_i) \\
    c_{i+1} &= \text{PairNorm}(\text{MLP}_c(c_i,\; e_i)) \\ 
    v_{i+1} &= \text{PairNorm}(\text{MLP}_v(v_i,\; A_{p} c_{i+1},\;  A_{n} c_{i+1},\; \nabla_{q_i}e_i)) \\
    out &= \sigma (\text{MLP}_o(v_{i+1})).
\end{align*}

At the beginning an empty state vector initialized with all ones is allocated for each variable and each clause and the unit is recurrently applied to them $s_{train}$ steps at training and $s_{test}$ steps at evaluation. At each step $i$ starting from $i=0$, QuerySAT produces a query $q_i \in \{0,1\}^{n \times d}$, where $d$ is the feature map count, by applying a  2-layer multi-layer perceptron $\text{MLP}_q$ to the variables state $v_i \in \mathbb{R}^{n \times d}$. A random noise vector $t \in \mathbb{R}^{n \times r}$ is also given in the input to this MLP. The query is range-limited by the sigmoid function $\sigma$ and evaluated by our unsupervised loss. Here we employ per-clause loss defined as $\mathcal{V}_{\phi}(q_i)=\{V_c(q_i))\; | \; c \in \phi\}$ which returns the vector of clause losses $e_i$. 

We then obtain the new clauses state $c_{i+1} \in \mathbb{R}^{m \times d}$ by applying a 2-layer perceptron $\text{MLP}_c$ and PairNorm \cite{zhao2019pairnorm} to the previous clauses state $c_i$ and the query result $e_i$.  Query mechanism thus replaces a variables-to-clauses message passing step that would be present in a classical message-passing architecture.
Then, information aggregation from clauses-to-variables is performed by message-passing, which for some variable sums all the clause states in which the variable occurs. Positive and negative occurrences are treated separately and are implemented as a sparse matrix multiplication between the clause state $c_{i+1}$ and occurrence matrices $A_{p}$ and $A_{n}$.

The new variables state $v_{i+1} \in \mathbb{R}^{n \times d}$ is then obtained  by applying a 3-layer perceptron $\text{MLP}_v$ and PairNorm to the variables state $v_i$, aggregated clause messages and the query loss gradient with the respect to query $q_i$. The new variables state is then mapped to the output variable assignment $out \in [0,1]^n$ using another 2-layer perceptron $\text{MLP}_o$ and sigmoid function $\sigma$. In all MLP layers we use LeakyReLU activation. 

The model is trained using the loss function $\mathcal{L}_{\phi}^{\log}$ proposed in section \ref{sec:sat_loss}. The loss is calculated from variable assignments $out$ at each step, and the sum of all losses is minimized. Using the loss at each time-step has shown performance improvements \cite{palm2017recurrent, amizadeh2019pdp, ozolins2020matrix} versus a single loss calculation at the end. Also, it enables using many more steps in evaluation than in training. Without impacting the model's performance, we use a conditional exit at any step if the variable assignment satisfies the input formula.

Although giving an additional noise parameter $t$ to $\text{MLP}_q$ seems like a minor detail, it serves several purposes. First, it avoids a blowup in the gradient during training, which may result from applying normalization to zero values. Secondly, it enables the model to learn a randomized algorithm by providing greater diversity between the queries, which is especially needed for obtaining good accuracy in testing with many steps. Thirdly, it allows the model to perform differently in several evaluations (we observed that in the experiments), mimicking restarts in the classical SAT solvers. We find that for our purposes $t \in \mathcal{N}(0,1)^{n \times r}$, where $r=4$, works well.

\subsection{Training tricks}\label{sec:tricks}
In this section, we describe the model architecture and training tricks that are not per se important for the power of the proposed architecture yet slightly improves model performance.

\textbf{Gradient scaling.} As we calculate loss and minimize it at each step, large gradient values may accumulate in the backwards pass for the first-time steps. Such uneven gradient distribution slows down training, and to mitigate that, we introduce gradient scaling and apply it to the variables and clauses states. It works by downscaling the gradient in the backward pass at each time-step as follows: $stop\_gradient(x)\alpha + x(1-\alpha)$, where $x$ is value and $\alpha \in [0,1]$ is a hyperparameter. We find that for QuerySAT with 32 recurrent steps, $\alpha = 0.2$ works the best.

\textbf{Multi-assignment loss.} Instead of returning a single variable assignment, we let the model return several $out \in [0,1]^{n \times u}$, where $u \in \mathbb{Z}^+$ is hyperparameter. The $u=8$ works best for us. For each assignment, we calculate loss value using the proposed loss function $\mathcal{L}_{\phi}^{\log}$ and then obtain the final step loss as a weighted sum of assignment losses. Weighting is done as follows - we sort assignments losses in descending order and enumerate them from $1$ to $u$. Let's call each such number the loss index. Then we calculate the final loss as the sum of each loss value multiplied with its squared index and then dividing the sum with the sum of squared indices. From our observations, such multi-assignment loss only gives improvements when used with unsupervised loss as the model in the early stages of training is free to explore various outputs. The assignment with the lowest loss value (before weighting) is promoted as the final model assignment. 

\subsection{Evaluation}\label{sec:querysat_evaluation}

\begin{table*}[ht]
\centering
\caption{Mean test accuracy (higher is better) as per cent of fully solved instances from the test set over 3 consecutive runs. Both models were trained with 32 recurrent steps for 500k training iterations and then tested with 32, 512, and 4096 recurrent steps. Value after $\pm$ indicates the standard error.}
\label{tab:sat_tasks_results}
\begin{adjustbox}{max width=\textwidth}
\begin{tabular}{lcccccc} 
\toprule
\multirow{2}{*}{Task} & \multicolumn{3}{c}{QuerySAT} & \multicolumn{3}{c}{NeuroCore}  \\ 
\cmidrule(lr){2-4}\cmidrule(lr){5-7}
                                       & $s_{test} = 32$ & $s_{test} = 512 $ & $s_{test}=4096$         & $s_{test} =32$ &$s_{test} = 512$ & $s_{test} = 4096$           \\ 
\toprule
k-SAT                                  &$72.12 \pm 0.19$    &$96.61 \pm 0.78$     &$\boldsymbol{99.05} \pm 0.38$            &$21.64 \pm 0.27$   &$46.85 \pm 5.02$    &$50.82 \pm 6.41$              \\ 
\midrule
3-SAT                                  &$61.89 \pm 5.19$    &$88.20 \pm 4.01$     &$\boldsymbol{93.32} \pm 3.21$            &$28.38 \pm 3.24$   &$53.49 \pm 3.94$     &$57.63 \pm 4.38$           \\ 
\midrule
3-Clique                               &$82.00 \pm 4.73$    &$93.06 \pm 4.67$     &$\boldsymbol{94.74} \pm 4.62$                  &$1.03 \pm 0.69$    &$1.03 \pm 0.66$     &$1.04 \pm 0.66$               \\ 
\midrule
k-Coloring           &$91.70 \pm 1.01$    &$97.76 \pm 0.98$     &$\boldsymbol{98.32} \pm 0.82$                  &$0.0 \pm 0.0$    &$0.0 \pm 0.0$     &$0.0 \pm 0.0$               \\ 
\midrule
SHA-1                  &$33.25 \pm 4.17$    &$\boldsymbol{46.57} \pm 1.16$     &$46.45 \pm 1.10$                  &$0.00 \pm 0.0$    &$0.27\pm 0.09$     &$0.24 \pm 0.09$            \\
\bottomrule
\end{tabular}
\end{adjustbox}
\end{table*}

We evaluate the QuerySAT model on several SAT tasks, which are standard benchmarks for classical SAT solvers, and the majority of them are not an easy feat for neural solvers. Namely, we chose k-SAT, 3-SAT, and also 3-Clique, k-Coloring, and SHA-1 preimage attack problems represented as CNF Boolean formulas. All datasets consists only of satisfiable formulas. The QuerySAT architecture is compared to the derivative of NeuroSAT \cite{selsam2018learning} that was used for predicting unsatisfiable cores by \cite{selsam2019guiding} but is generally applicable to any variables-wise predictions on SAT factor graph. Further on, we refer to this NeuroSAT derivative as NeuroCore. The results are also compared to the GSAT and Glucose 4 classical solvers. For all tasks, we generate a train set of 100k formulas and validation and test sets of 10k formulas each. The 3-SAT and SHA-1 tasks are an exception as their test sets consist of 40k and 5k formulas, respectively. For most tasks, we use larger formulas in the test set to evaluate the generalization capability of the model.

The k-SAT task is taken from \cite{selsam2018learning} -- each clause in the $n$ variable formula is generated by sampling a small integer $k$ (size of the clause), and then randomly without replacement taking $k$ of the $n$ variables. Each variable in the clause is then negated with 50\% probability. The resulting clauses consist of roughly four variables on average. The train and validation sets consist of formulas with 3 to 100 variables, but the test set -- with 3 to 200 variables.

Graph and 3-SAT tasks are generated using the CNFGen library \cite{lauria2017cnfgen}, which allows encoding several popular problems as SAT instances. We generate hard 3-SAT instances at the satisfiability boundary where the relationship between the number of clauses ($m$) and variables ($n$) is  $m = 4.258n + 58.26n^{-\frac{2}{3}}$ \cite{crawford1996experimental}. The train and validation sets of 3-SAT tasks consist of formulas with 5 to 100 variables, but test set -- with 5 to 405 variables. 

For the graph-based tasks, we generate  Erdős–Rényi graphs with edge probability $p$. For the 3-Clique task, where the main goal is to find all triangles in the graph, we use $p=3 ^ \frac{1}{3} / (v(2 - 3v + v^2))^ {\frac{1}{3}}$, where $v$ is the vertex count in the graph. For the k-Coloring task we use $p = \frac{(1 + 0.2) \ln v}{v} + 0.05$ and the goal is to color the graph with at least $k$ colours. Such generation produces mostly sparse connected graphs that are colourable using 3 to 5 colours. The graphs are then encoded as SAT instances using the CNFGen library. For both tasks, train and validation sets consist of graphs with 4 to 40 vertices, but the test set -- with 4 to 100 vertices.

We also experiment with the SHA-1 preimage attack task from the SAT Race 2019 \cite{skladanivskyy_minimalistic_2019}. The goal of this task is to find the message value given the hash value. The original SAT competition task uses the SHA-1 algorithm with 17 rounds and asks the solver to find the first up to 160 message bits. Such configuration produces SAT instances at the threshold of satisfiability, and it is expected that only a single solution exists. That is a challenging task even for modern SAT solvers; hence, we use a smaller set-up to find the first 2 to 20 bits of the message. But even then, it is still a moderately hard task. We generate instances for this task using the CGen tool \cite{skladanivskyy2020tailored}.

To make the QuerySAT and NeuroCore architectures comparable, we train NeuroCore with the same per-step loss and apply the same bag of tricks that we used for QuerySAT, as described in Sections \ref{sec:querysat_model} and \ref{sec:tricks}. NeuroCore is also given a chance to return results in any step if the correct solution has been found. The rest of the architecture is left intact.
For both models, we use 128 feature maps that produce similarly sized models and train them with a batch size of 20000 nodes (max node count in the input factor graph), 32 recurrent steps and a learning rate $2 \times 10^{-4}$. Hyperparameters are selected by performing a grid search by hand. On the SHA-1 preimage attack, models are trained for 1M iterations, but on the rest of the tasks for 500k iterations. Afterwards, models are tested with 32, 512, and 4096 recurrent steps and the mean accuracy as a percentage of fully solved instances over 3 consecutive runs is depicted in Table \ref{tab:sat_tasks_results}. We see that QuerySAT with 4096 steps performs consistently the best for all the tasks. It can solve 3-Clique, k-Coloring and SHA-1 tasks on which NeuroCore produces an accuracy of almost zero.

Detailed comparison of both architectures is conducted on the 3-SAT task, where we check step-wise and formula-wise generalization of both models. To evaluate both properties, we use models trained with 32 recurrent steps on formulas with 3 to 100 variables. Fig. \ref{fig:accuracy_on_datasets} depicts generalization to harder formulas (up to 400 variables) with 32 and 4096 recurrent steps at the test time. Fig. \ref{fig:accuracy_on_steps} shows step-wise generalization by changing the number of steps $s_{test}$ from 4 up to 131k when testing on the same formulas with 400 variables. QuerySAT outperforms NeuroCore on both generalization tasks by a wide margin and, with 32 steps, has a similar performance to the NeuroCore with 4096 steps. QuerySAT's performance increases with the step count at the test time, although it is trained only with 32 steps. NeuroCore, on the contrary, plateaus at approximately 16k steps and 40\% accuracy. In Fig. \ref{fig:neural_solver_cactus}, the cactus plot is shown comparing NeuroCore, and QuerySAT tested with 4096 recurrent steps on 3-SAT instances with 5 to 405 variables. Cactus plot shows cumulative time spend for solving instances (y-axis) versus cumulative solved instance count (x-axis). More solved instances in less time indicate better model performance. At the same time interval, QuerySAT solves $\sim 36k$ formulas, while NeuroCore solves only $\sim 18k$ from the total of 40k 3-SAT formulas in the test set (see Fig. \ref{fig:neural_solver_cactus}).

\begin{figure*}[ht]
     \centering
     \begin{minipage}[t]{0.48\textwidth}
         \centering
         \includegraphics[width=0.77\textwidth]{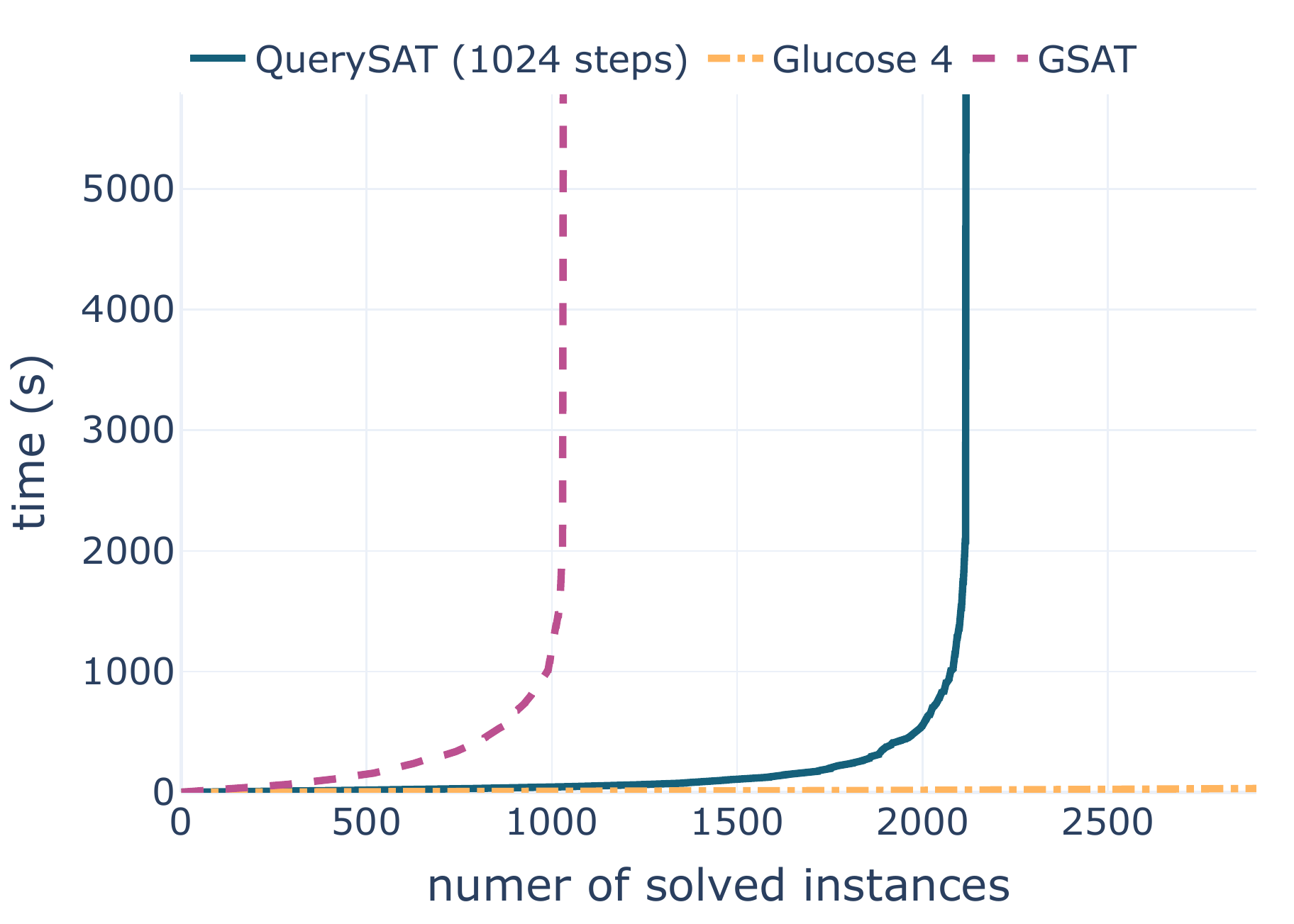}
         \caption{Cactus plot representing the number of solved SHA-1 preimage attack instances vs. used time for QuerySAT, GSAT and Glucose 4 solvers.}
         \label{fig:sha1_query_vs_gsat}
     \end{minipage}
     \hfill
     \begin{minipage}[t]{0.48\textwidth}
         \centering
         \includegraphics[width=0.77\textwidth]{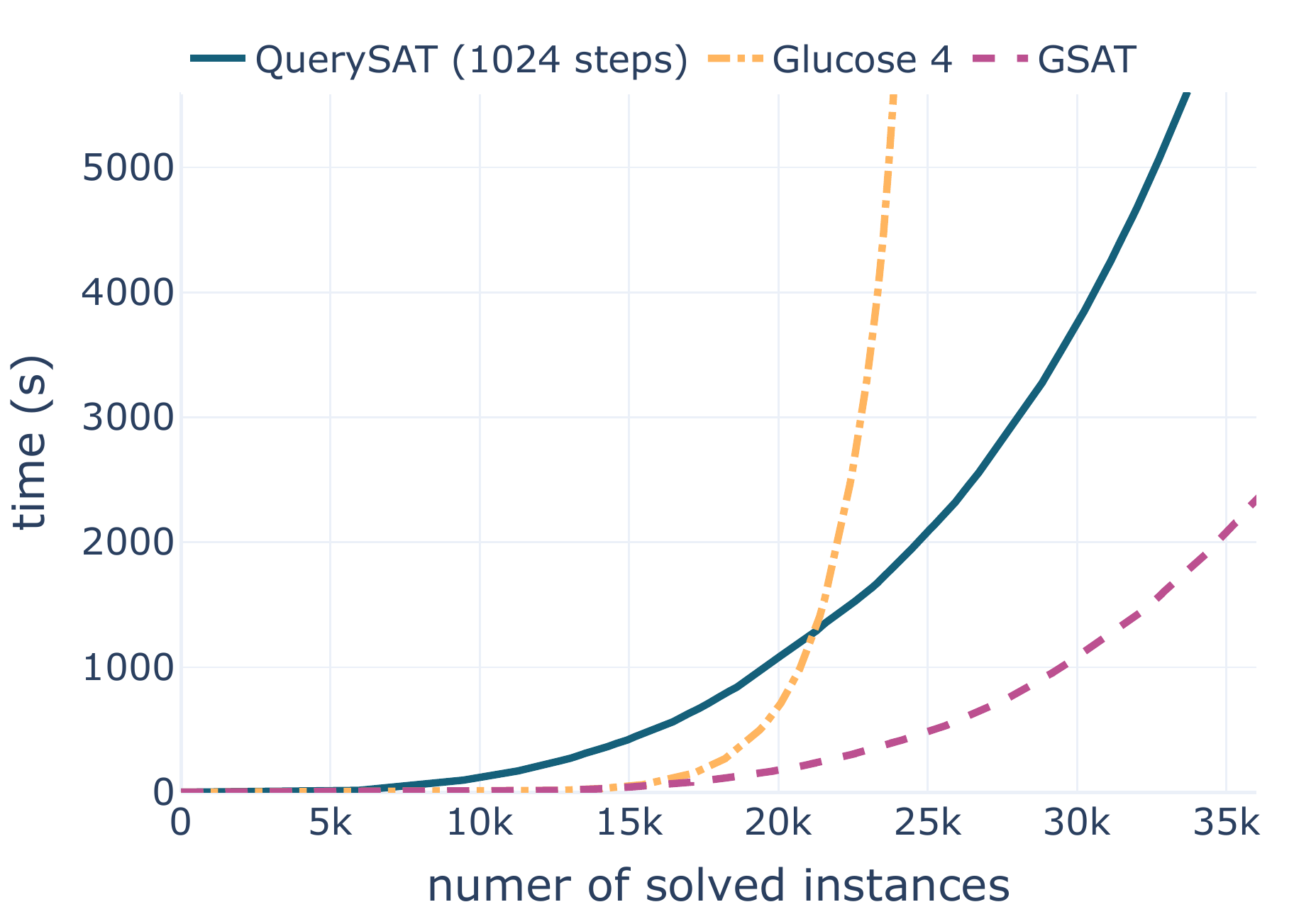}
         \caption{Cactus plot representing the number of solved 3-SAT instances vs. used time for QuerySAT, GSAT and Glucose 4 solvers.}
         \label{fig:scatter_query_vs_gsat}
     \end{minipage}
     \begin{minipage}[t]{0.48\textwidth}
         \centering
         \includegraphics[width=0.77\textwidth]{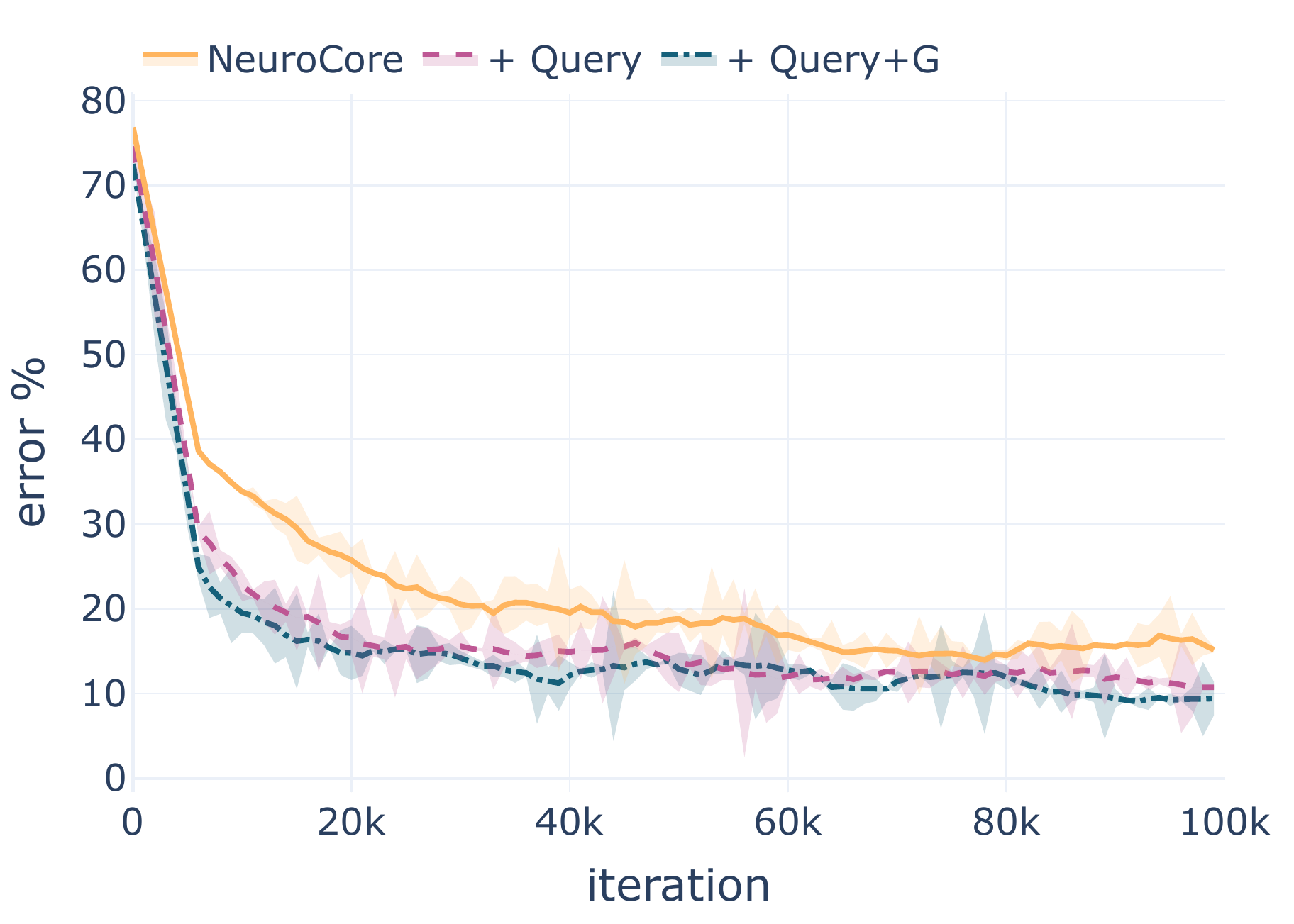}
         \caption{Error as per cent of unsolved instances in the validation set depending on training iteration for the k-SAT task. The validation set consists of 10k formulas with 3-100 variables. $s_{test}=64$ is used for validation.}
         \label{fig:training_k-sat}
     \end{minipage}
     \hfill
     \begin{minipage}[t]{0.48\textwidth}
        \centering
        \includegraphics[width=0.77\textwidth]{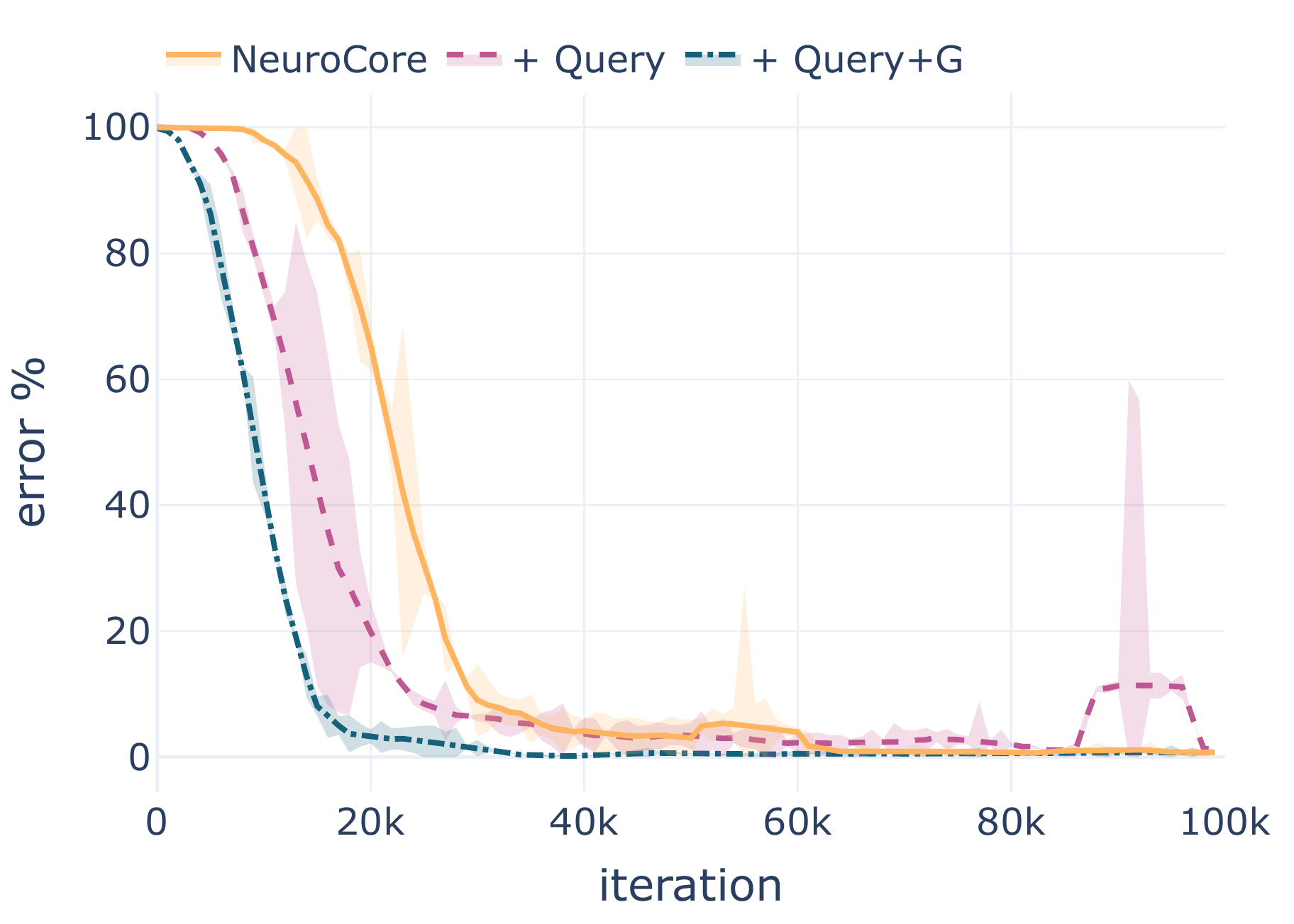}
        \caption{Error as per cent of unsolved instances in the validation set depending on training iteration for the 3-Clique task. The validation set consists of 10k graphs with 4-20 vertices. $s_{test}=64$ is used for validation.}
        \label{fig:training_k-clique}
     \end{minipage}
\end{figure*}

To show current capabilities of QuerySAT, we also compare it to GSAT \cite{selman1992a} and Glucose 4 \cite{audemard2018glucose,een2003extensible} classical solvers. GSAT is a widely known incomplete local-search solver, and Glucose is a contemporary conflict-driven clause learning solver. As the QuerySAT is also an incomplete solver, it's directly comparable to the GSAT algorithm. Comparison to Glucose is not on equal ground since Glucose can certify that some instance is UNSAT while QuerySAT runs indefinitely for such instance. That said, we present their comparison nonetheless to give a rough idea of their relative performance. Note that both - GSAT and Glucose 4 - may not accurately represent the performance of modern state-of-the-art SAT solvers. All solvers are evaluated on the same hardware but note that QuerySAT utilizes a single Nvidia T4 GPU while the others do not. We configure all three solvers to have approximately a 2-second time limit by giving QuerySAT 1024 recurrent steps, GSAT -- 500k steps, and setting a 2-second timeout for Glucose solver. We compare these solvers on previously described test sets of 3-SAT and SHA-1 tasks. The QuerySAT was trained on the forenamed train sets. The results are depicted as cactus plots in the Fig. \ref{fig:sha1_query_vs_gsat} and \ref{fig:scatter_query_vs_gsat}. Glucose utilizes the structure of the problem and therefore struggles on random 3-SAT instances yet solves the SHA-1 task in almost no time. Even though GSAT outperforms QuerySAT on 3-SAT instances, QuerySAT seems to utilize the SHA-1 structure and achieves better performance than GSAT. QuerySAT, similarly to other contemporary end-to-end neural solvers \cite{selsam2018learning, amizadeh2019pdp}, in the general case, is outperformed by classical solvers and requires breakthroughs to allow scaling them to large industrial instances.



\section{Evaluating the query mechanism}\label{sec:eval_query_mechanism}

\begin{table*}[ht]
\centering
\caption{Mean accuracy of NeuroCore and its augmentation with query and gradient and as per cent of fully solved instances from the test set over 3 consecutive runs. All models are trained with 32 recurrent steps and evaluated with 32, 512, and 4096 steps. Value after $\pm$ indicates the standard error. Note that the training, validation and test instances are smaller than in Table \ref{tab:sat_tasks_results}. 
}
\label{tab:query_generalization}
\begin{adjustbox}{max width=\textwidth}
\begin{tabular}{lcccccc} 
\toprule
\multirow{2}{*}{}                        & \multicolumn{3}{c}{k-SAT  }                                                 & \multicolumn{3}{c}{3-Clique}              \\ 
\cmidrule(lr){2-4}\cmidrule(lr){5-7}
                                              & \multicolumn{1}{c}{$s_{test} = 32$} & \multicolumn{1}{c}{$s_{test} = 512$} & \multicolumn{1}{c}{$s_{test} = 4096$} & \multicolumn{1}{c}{$s_{test} = 32$} & \multicolumn{1}{c}{$s_{test} = 512$} & \multicolumn{1}{c}{$s_{test} = 4096$} \\ 
\toprule
\textit{NeuroCore}
&\begin{tabular}{@{}c@{}}$38.86 \pm 1.04$\end{tabular}
&\begin{tabular}{@{}c@{}}$56.14 \pm 3.48$\end{tabular}
&\begin{tabular}{@{}c@{}}$60.97 \pm 6.19$\end{tabular}
&\begin{tabular}{@{}c@{}}$65.38 \pm 5.09$\end{tabular}
&\begin{tabular}{@{}c@{}}$67.12 \pm 3.25$\end{tabular} 
&\begin{tabular}{@{}c@{}}$70.39 \pm 2.68$\end{tabular} \\
\midrule
\textit{\, + Query}
&\begin{tabular}{@{}c@{}}$41.09 \pm 2.33$\end{tabular}
&\begin{tabular}{@{}c@{}}$61.21 \pm 9.89$\end{tabular}
&\begin{tabular}{@{}c@{}}$67.21 \pm 13.54$\end{tabular}
&\begin{tabular}{@{}c@{}}$62 .90 \pm 2.55$\end{tabular}
&\begin{tabular}{@{}c@{}}$78.19 \pm 2.67$\end{tabular}
&\begin{tabular}{@{}c@{}}$84.55 \pm 2.93$\end{tabular} \\
\midrule
\textit{\, + Query + G}
&\begin{tabular}{@{}c@{}}$\boldsymbol{47.92 \pm 1.66}$\end{tabular}
&\begin{tabular}{@{}c@{}}$\boldsymbol{71.95 \pm 6.59}$\end{tabular}
&\begin{tabular}{@{}c@{}}$\boldsymbol{75.50  \pm 7.15}$\end{tabular}
&\begin{tabular}{@{}c@{}}$\boldsymbol{86.13 \pm 4.83}$\end{tabular}
&\begin{tabular}{@{}c@{}}$\boldsymbol{93.96\pm 2.59}$\end{tabular}
&\begin{tabular}{@{}c@{}}$\boldsymbol{95.50\pm 1.95}$\end{tabular} \\
\bottomrule
\end{tabular}
\end{adjustbox}
\end{table*}


We evaluate the impact of the query mechanism by augmenting NeuroCore architecture with it and measuring the improvement. It is straightforward to do since NeuroCore is similar to QuerySAT. NeuroCore uses literals-to-clauses and clauses-to-literals message-passing to update the internal states for literals and clauses. Therefore, we add a query mechanism alongside the literals-to-clauses message-passing and also give the gradient of the evaluated query to the MLP that updates the literals state.  We experiment with two variants of the query mechanism: with both query and gradient (\textit{+ Query + G}) and only with query (\textit{+ Query}). Both versions are compared with the standard NeuroCore architecture (\textit{NeuroCore}).

We chose k-SAT and 3-Clique tasks (see section \ref{sec:querysat_evaluation}) for evaluation. We use the same dataset split as previously. For the k-SAT task, we use the same train and validation set as previously, but for the test set generate formulas with 100 to 200 variables. On the same note, train and validation sets for the 3-Clique task consists of graphs with 4 to 20 vertices but the test set of 20 to 40 vertices. The test set consists of harder formulas to see how various variants of query mechanism impacts generalization. 


All three versions are trained for 100k iterations with the same training and network configuration as described in Section \ref{sec:querysat_evaluation}. The trained model is then tested with 32, 512, and 4096 recurrent steps. The mean results of 3 consecutive runs are depicted in Table \ref{tab:query_generalization}. We also include validation error in the train time for each version using 64 recurrent steps ($s_{test}$), it is depicted in the Fig. \ref{fig:training_k-sat} and \ref{fig:training_k-clique}.

In Table \ref{tab:query_generalization}, we can see that both versions with a query mechanism outperform the NeuroCore baseline. The version with a query mechanism and gradient trains faster and achieves better accuracy. Interestingly, adding a gradient significantly improves the model's accuracy of the 3-Clique task. 
We reason that gradient is helpful for the tasks that represented as a SAT instance has a distinct structure. Such structure is very expressive for the 3-Clique task, yet, on the contrary, k-SAT doesn't have any distinct structure as it is sampled from a uniform distribution \cite{yolcu2019learning}.

\section{Related work}
Neural networks have been proposed as an effective alternative for automatically developing heuristic algorithms for NP-hard problems \cite{bengio2020machine, cappart2021combinatorial}. Two main research directions are replacing handcrafted heuristics with a neural network in a classical solver and building end-to-end neural solvers.

\cite{selsam2018learning} proposed a Graph Neural Network (GNN) architecture called NeuroSAT that is trained using single-bit supervision to predict whether the Boolean formula is satisfiable or not. The variable assignment is obtained from the last layer by performing  2-clustering on the output, but the network is never optimized for producing the assignment explicitly. Hence, it is not clear why such an approach should work.  In a later work, \cite{selsam2019guiding} simplified NeuroSAT architecture (they call it NeuroCore)  and used it for guiding high-performance solvers (e.g.,  MiniSAT, Glucose) by predicting how likely each variable is in the unsatisfiable core.  \cite{kurin2019improving} use Q-learning to train similar GNN for predicting branching heuristics in MiniSAT solver. Others \cite{azar2020nngsat, kurin2019improving} similarly augment classical solvers with neural networks to solve 2-Quantified Boolean Formulas and logic locked circuits.  On the same trend, \cite{yolcu2019learning} trained a GNN using reinforcement learning to learn a local-search heuristic that finds a solution similar to the GSAT / WalkSAT algorithm by flipping a single variable in each step. It is also theoretically shown that GNN can learn to mimic the WalkSAT algorithm \cite{chen2019graph}. 

Supervised and reinforcement learning are not the best options for solving SAT, as the variable assignment is time-consuming to obtain, many possible solutions exist, and training may be slow. \cite{amizadeh2018learning} proposed a differentiable unsupervised approach for directly solving Circuit-SAT. For this purpose, they represent variables as real numbers in the range $[0,1]$ and use GNN for producing variable predictions on the Circuit-SAT graph. The predicted variables are evaluated as Circuit-SAT, but the AND and OR functions are substituted with differentiable Softmax and Softmin functions. In a later work \cite{amizadeh2019pdp}, they have applied the same method for solving SAT problems. To encourage exploration, they introduce a temperature parameter that is annealed in the training time, making training cumbersome. Although the training method is closely related to ours, their loss function is more complex.  \cite{kyrillidis2020fouriersat} lately proposed to represent Boolean functions by multilinear polynomials and find a solution to a single formula by using gradient descent optimization. Even though their loss function is similar to ours, they solve the MAX-SAT problem and do not directly optimize towards finding a solution to the SAT problem.


\section{Conclusions}
In this paper, we have proposed a query mechanism that allows the neural network to make several solution trials, obtain the loss of each trial and change its strategy accordingly. To evaluate the impact of the query mechanism, we propose an unsupervised SAT loss and integrate it with queries to form the QuerySAT architecture. We find that QuerySAT outperforms the message-passing neural baseline on all proposed tasks: k-SAT, 3-SAT, 3-Clique, k-Coloring, and SHA-1 preimage attack. Experiments show that query mechanism can significantly increase the performance of message-passing graph neural networks. To give a better insight into the current capabilities of QuerySAT, we also compare it with classical solvers on 3-SAT and SHA-1 preimage attack tasks. Although we have analyzed only SAT solving, we expect a similar benefit from the query mechanism on other neural architectures and tasks. Since QuerySAT employs unsupervised loss, not requiring to know the labels, it provides rich opportunities for integration with classical solvers.

\section*{Acknowledgment}

 We would like to thank the IMCS UL Scientific Cloud for the computing power and Leo Trukšāns for the technical support. This research is supported by Google Cloud and funded by the Latvian Council of Science, projects No.~lzp\nobreakdash-2018/1\nobreakdash-0327, lzp-2021/1-0479.

\bibliographystyle{IEEEtran}
\bibliography{bibliography.bib}

\clearpage
\begin{appendices}

\section{Proof of Theorem \ref{thm:int_points}}\label{apx:int_points}

\thintpoints*
\begin{proof}
The idea is to give proof in two steps. First, we prove that when a binary query satisfies the formula, then formula value $\mathcal{L}_{\phi}(x)$ and all clauses values $V_c(x)$ should be 1. Secondly, when a binary query does not satisfy the formula, at least one clause value and formula value should be 0.

Firstly, assume that the query corresponds to a satisfying assignment, then all clauses are satisfied. That implies that in each clause, there is at least one positive literal with its query value being 1 or a negative literal with its query value being 0. Therefore in each clause loss $V_c(x) = 1 - \prod_{i \in c^+}(1-x_i)\prod_{i \in c^-}x_i$ at least one of the elements of the product is 0. That implies that all clause losses $V_c(x)$ are equal to 1. Therefore the loss $\mathcal{L_{\phi}}(x) = \prod_{c \in \phi}V_c(x)$ is 1.

Secondly, assume the query corresponds to an unsatisfying assignment, then there is at least one clause $c$ that is unsatisfied. This implies that in the unsatisfied clause $c$ the query value for all positive literals is 0 and for all negative literals it is 1. Therefore in the clause loss $V_c(x) = 1 - \prod_{i \in c^+}(1-x_i)\prod_{i \in c^-}x_i$ all elements of the product are 1. This implies that for the unsatisfied clause the corresponding clause loss $V_c(x)$ is 0. Therefore the loss $\mathcal{L_{\phi}}(x) = \prod_{c \in \phi}V_c(x)$ is 0, since one of the elements of the product is 0.
\end{proof}

\section{Proof of Theorem \ref{thm:instance_to_be_solved}}\label{apx:instance_to_be_solved}

\thsolvedinstances*
\begin{proof}

The idea is to choose the query input values based on primes such that the clause loss contains an irreducible fraction with the numerator and denominator consisting of prime factors which point to the variables that appear in the clause. From the measured clause loss value, it is possible to infer the corresponding irreducible fraction, and from this fraction, we can infer the variables and the signs of their corresponding literals in the clause.

An element of the query $x=( x_{1} ,x_{2} ,\ldots ,x_{n}) \in \mathbb{R}^{n}$ corresponds to a value of a variable queried at an intermediate (real) point between $0$ and $1$, where $n$ is the number of variables. Each query element is obtained by using a pair from the series of prime pairs where the gap between the two primes in the pair is a fixed constant $H$. For $H=2$ the series would be pairs of twin primes: $(3,5), (5,7), (11,13), \ldots $ . It has been proven that there is a constant $H$ between 2 and 246 such that the corresponding prime pair series contains an infinite number of elements \cite{polymath2014variants}. We fix a concrete value of $H$ for constructing the query.

For constructing the query from the series, we take the first $n$ pairs $( a_{1} ,b_{1}) ,( a_{2} ,b_{2}) ,...,( a_{n} ,b_{n})$ that fulfill the following criteria: $\forall _{i} \ a_{i} >H$ and $\forall _{i,j} \ a_{i} \neq b_{j}$. Note that $b_{i}-H = a_{i}$. We set the query $x$ as $( H/b_{1} ,H/b_{2} ,...,H/b_{n})$.

Let us examine the clause loss $V_c(x) = 1 - \prod_{i \in c^+}(1-x_i)\prod_{i \in c^-}x_i$ corresponding to a clause $c$ and the query $x$ with $i$ positive and $k-i$ negative literals. To simplify the notation, we assume without a loss of generality that the first $k$ variables appear in the clause and that the positive literals have lower indices than negative literals.

\begin{equation*}
    \scalemath{0.75}{
    \begin{aligned}
        V_c(x) &=1-(1-x_{1})(1-x_{2}) \ldots (1-x_{i})x_{i+1}x_{i+2} \ldots x_{k} = \\
        &=1-\left(1-\frac{H}{b_{1}}\right)\left(1-\frac{H}{b_{2}}\right) \ldots \left(1-\frac{H}{b_{i}}\right)\frac{H}{b_{i+1}}\frac{H}{b_{i+2}} \ldots \frac{H}{b_{k}} = \\
        &=1-\frac{b_{1} -H}{b_{1}}\frac{b_{2} -H}{b_{2}} \ldots \frac{b_{i} -H}{b_{i}}\frac{H}{b_{i+1}}\frac{H}{b_{i+2}} \ldots \frac{H}{b_{k}} = \\
        &=1-\frac{a_{1}}{b_{1}}\frac{a_{2}}{b_{2}} \ldots \frac{a_{i}}{b_{i}}\frac{H}{b_{i+1}}\frac{H}{b_{i+2}} \ldots \frac{H}{b_{k}} =1-\frac{a_{1} a_{2} \ldots a_{i} H^{k-i}}{b_{1} b_{2} \ldots b_{k}}
    \end{aligned}
    }
\end{equation*}

The indices of the primes $b_{1} b_{2} \ldots b_{k}$ in the denominator of the fraction in the obtained expression correspond to the indices of variables appearing in the clause $c$. The indices of the primes $a_{1} a_{2} ...a_{i}$ in the numerator correspond to the indices of variables appearing as positive literals in the clause.

The primes $a_{1} a_{2} \ldots a_{i}$ and $b_{1} b_{2} \ldots b_{k}$ are non-overlapping and the prime factors of $H$ are smaller than any prime $a_{i}$ or $b_{i}$ due to the pair selection criteria. This implies that the fraction in the representation of the clause loss that we obtained is irreducible.

We prove that given the constructed query $x$, the clause loss function is injective. To show that, we take a clause $c'$ that is different from our arbitrarily chosen clause $c$. The clause $c'$ being different from clause $c$ implies that it will differ in at least one literal, and therefore the irreducible fraction in the corresponding clause losses $V_c'(x)$ and $V_c(x)$ will differ in $a_{1} a_{2} \ldots a_{i}$ or $b_{1} b_{2} \ldots b_{k}$. Due to the fundamental theorem of arithmetic, there is only one way to write a number as a product of its prime factors (up to the order of the factors). The numerator or the denominator of the fraction in the losses will differ because their prime factors will differ. Therefore any two different clauses will produce clause losses that each contain a different irreducible fraction. Since each rational number can be written as an irreducible fraction in exactly one way, the two clause losses are different rational numbers, and hence the clause loss function is injective for the constructed query.

Therefore the clause can be identified from a query and its corresponding loss for the clause. Since the instance $\phi$ to be solved is uniquely represented by its constituent clauses, it can be identified from a query and a set containing a clause loss for each of its clauses.
\end{proof}

\section{Query results}\label{apx:query_use}

\begin{figure*}[ht]
     \centering
     \begin{minipage}[t]{0.32\textwidth}
        \centering
        \includegraphics[width=\textwidth]{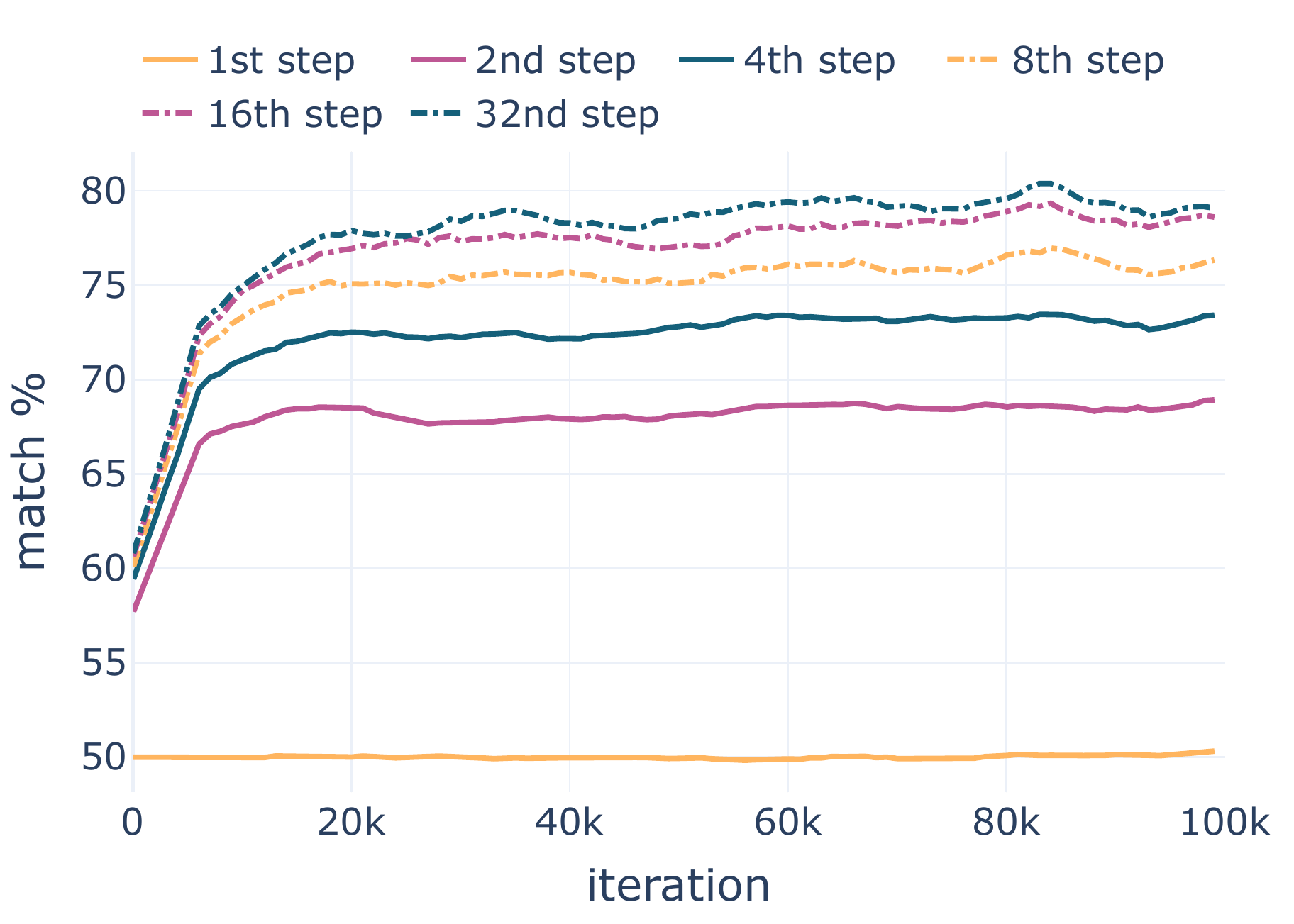}
        \caption{Match between query and logits (higher value means more similar) of the same recurrent step depending on the training iteration on the 3-SAT task.}
        \label{fig:query_vs_logits}
     \end{minipage}
     \hfill
     \begin{minipage}[t]{0.32\textwidth}
        \centering
        \includegraphics[width=\textwidth]{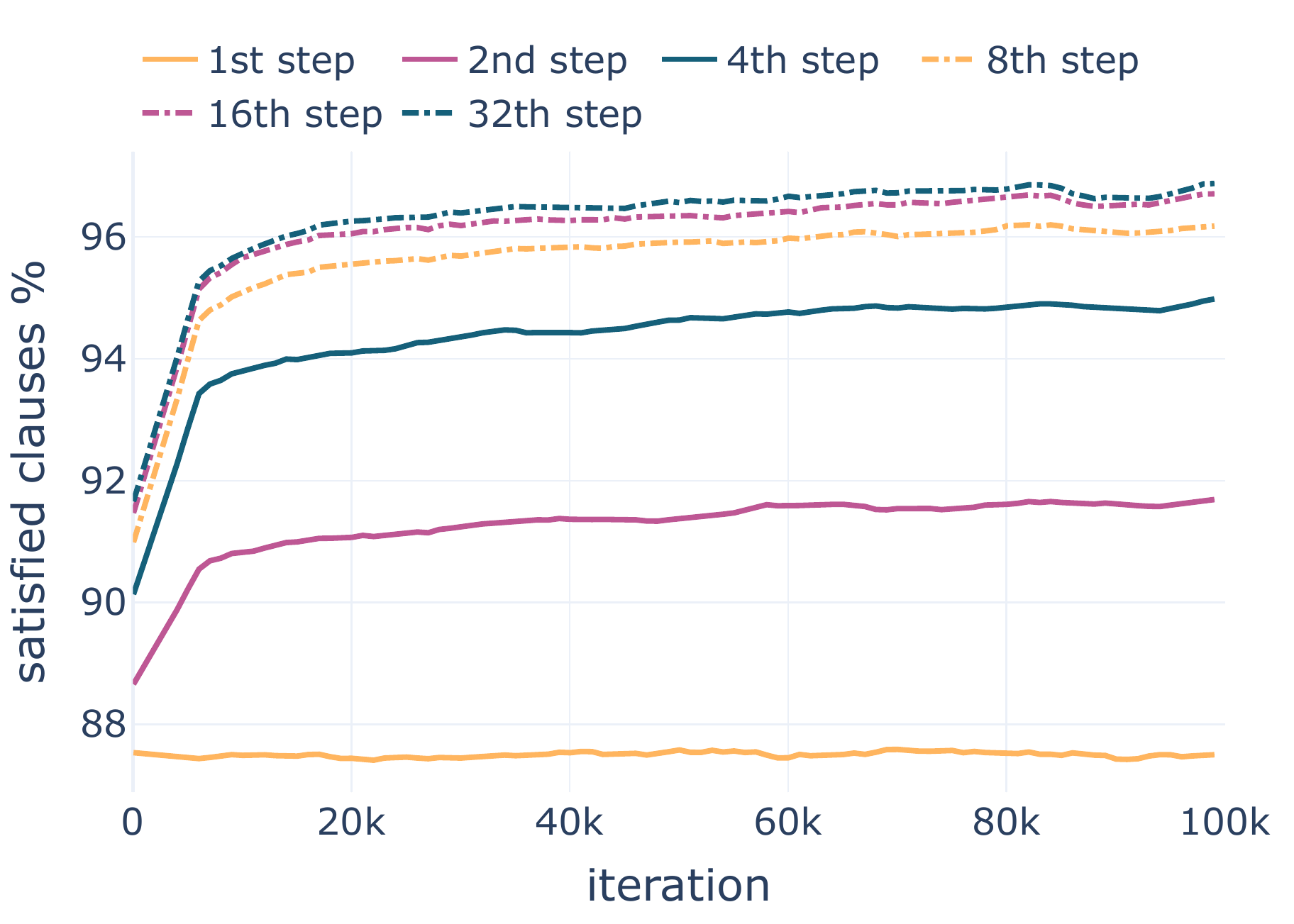}
        \caption{Per cent of clauses satisfied by query depending on the training iteration on the 3-SAT task.}
        \label{fig:sat_clauses}
     \end{minipage}
     \hfill
     \begin{minipage}[t]{0.32\textwidth}
        \centering
        \includegraphics[width=\textwidth]{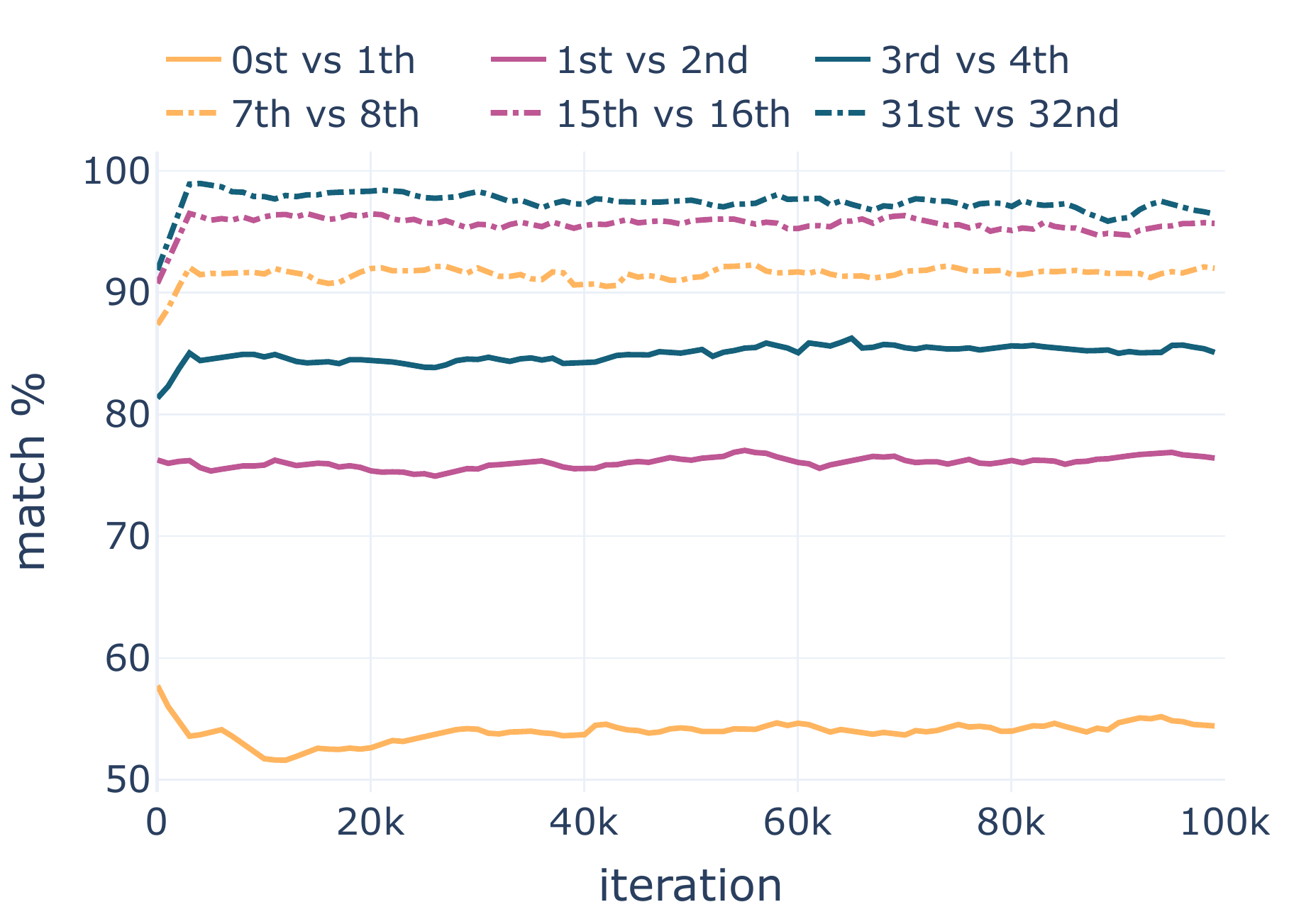}
        \caption{Element-wise match (higher values indicate greater similarity) between two queries in consecutive steps depending on the training iteration on the 3-SAT task.}
        \label{fig:queries_diff}
     \end{minipage}
\end{figure*}

To analyze how the model uses queries and what information they contain, we train NeuroCore with the query mechanism on the 3-SAT task for 100k iterations as described in the section \ref{sec:eval_query_mechanism}.  At each recurrent step, query and logit values are discretized by applying the sigmoid function and then rounding values to the closest integer (0 or 1).  Fig. \ref{fig:sat_clauses} reveals how many clauses the discretized query satisfies at each step in the training time. Fig. \ref{fig:query_vs_logits}, on the other hand, depicts an element-wise match between discretized queries and logits as per cent of identical variables from the total variable count. In Fig. \ref{fig:query_vs_logits}, we can see that the queries have different values from the logits of the same step and from the Fig. \ref{fig:sat_clauses} we can see that they never satisfy given Boolean formula, strengthening our belief that queries are not used to output the true prediction but instead to obtain rich information about the problem.

We also validate that model produces different queries at each recurrent step, therefore obtaining rich information about the problem. Fig. \ref{fig:queries_diff} depicts the match as per cent of equal elements from the total element count between two discretized queries of two consecutive steps. As reasoned in section \ref{sec:sat_loss}, query mechanism can extract more information about the problem by issuing several different queries at each step and accumulating this knowledge. In Fig. \ref{fig:queries_diff}, we can see the difference between the last and second last query increase with the training time. That indicates that the model learns to issue more different queries, increasing the obtained information about the problem.

\end{appendices}
\end{document}